\documentclass[lettersize,journal]{IEEEtran}
\usepackage{color}
\definecolor{pink}{rgb}{0.58,0,0.83}
\definecolor{orange}{rgb}{1,0.5,0}
\definecolor{lightgreen}{rgb}{0.2, 0.8, 0.2}
\definecolor{lightyellow}{rgb}{0.84, 0.65, 0.13}

\newcommand{\mc}{\textcolor{blue}}

\newcommand{\mt}{\textit}
\usepackage{amsmath,amsfonts}
\usepackage{algorithmic}
\usepackage[linesnumbered, ruled]{algorithm2e}
\usepackage{array}
\usepackage{subfigure}
\usepackage{makecell}
\usepackage[caption=false,font=normalsize,labelfont=sf,textfont=sf]{subfig}
\newenvironment{proof}{{\indent \indent \it Proof:\quad}}{\hfill $\blacksquare$\par}

\usepackage{textcomp}
\usepackage{stfloats}
\usepackage{url}
\usepackage{verbatim}
\usepackage{graphicx}
\usepackage{amssymb}
\usepackage{dsfont}
\usepackage{bbding}
\usepackage{booktabs}
\usepackage{multirow}
\usepackage{tablefootnote}
\usepackage{cite}
\usepackage{xcolor}
\usepackage{threeparttable}
 
\newtheorem{remark}{Remark}

\newtheorem{thm}{Theorem}
\usepackage{arydshln}
\usepackage{multicol}
\definecolor{ForestGreen}{RGB}{34,139,34}
\def\BibTeX{{\rm B\kern-.05em{\sc i\kern-.025em b}\kern-.08em
		T\kern-.1667em\lower.7ex\hbox{E}\kern-.125emX}}
\usepackage{balance}
\begin{document}
	\title{Learning-based Near-optimal Motion Planning for Intelligent Vehicles with Uncertain Dynamics}
	\author{Yang~Lu,
		Xinglong~Zhang,~\IEEEmembership{Member,~IEEE},
		Xin~Xu$^\star$,~\IEEEmembership{Senior Member,~IEEE},
		Weijia~Yao,~\IEEEmembership{Member,~IEEE}
		\thanks{Yang Lu, Xinglong Zhang, and Xin Xu are with the College of Intelligence Science and Technology, National University of Defense Technology, Changsha, Hunan, 40073 P.R. China. (e-mail: luyang18@mail.sdu.edu.cn, zhangxinglong18@nudt.edu.cn, xuxin$\_$mail@263.net).

		Weijia Yao is with the School of Robotics, Hunan University, Changsha, Hunan, 410082, P.R. China.
		}}
		
		\maketitle
		\begin{abstract}
			Motion planning has been an important research topic in achieving safe and flexible maneuvers for intelligent vehicles. However, it remains challenging to realize efficient and optimal planning in the presence of uncertain model dynamics. In this paper, a sparse kernel-based reinforcement learning (RL) algorithm with Gaussian Process (GP) Regression (called GP-SKRL) is proposed to achieve online adaption and near-optimal motion planning performance. In this algorithm, we design an efficient sparse GP regression method to learn the uncertain dynamics. Based on the updated model, a sparse kernel-based policy iteration algorithm with an exponential barrier function is designed to learn the near-optimal planning policies with the capability to avoid dynamic obstacles. Thereby, batch-mode GP-SKRL with online adaption capability can estimate the changing system dynamics. The converged RL policies are then deployed on vehicles efficiently under a safety-aware module. As a result, the produced driving actions are safe and less conservative, and the planning performance has been noticeably improved. Extensive simulation results show that GP-SKRL outperforms several advanced motion planning methods in terms of average cumulative cost, trajectory length, and task completion time. In particular, experiments on a Hongqi E-HS3 vehicle demonstrate that superior GP-SKRL provides a practical planning solution.
		\end{abstract}

\section{Introduction}
	Intelligent vehicles are promising in reducing traffic accidents due to human factors. As an important and indispensable part of autonomous driving, kinodynamic motion planning algorithms help improve the safety and maneuvering capabilities of intelligent vehicles by planning multistep-ahead trajectory profiles under the dynamic constraints and state limitations\cite{lavalle2001randomized,webb2013kinodynamic}. In complex and unstructured road scenarios, realizing optimal kinodynamic motion planning still faces challenges, caused by the varying tire-ground friction coefficient, uncertain model parameters, measurement noise, etc.

	Model predictive control (MPC) algorithms generate trajectories by solving optimal control problems under collision constraints\cite{zeng2021safety,brito2019model,liu2018convex}. Since MPC methods rely on an accurate model, uncertain vehicle dynamics could lead to performance degradation in planning and control. Therefore, it is necessary to estimate the model uncertainties online \cite{hewing2019cautious,hewing2018cautious,kabzan2019learning}. On the other hand, online solving nonlinear or non-convex optimization problems may lack the reliability to obtain feasible solutions and may be computationally intensive.
	
	RL and ADP have received significant attention in recent years for solving optimal planning and control problems\cite{lian2014motion,liu2013policy}. Previous RL approaches in solving planning problems with uncertain dynamics employ the integration of Pontryagin's Maximum Principle to address model uncertainty and	kinodynamic constraints\cite{kamthe2018data,he2021integral}, incorporate model learning strategy and barrier function into the cost function to separately handle model uncertainty and collision constraints\cite{9756946,mahmud2021safety}, measure and select the optimal solution among the planning candidates to address the performance degradation caused by model uncertainty\cite{hoerger2020non,bejjani2018planning}. RL approaches necessitate the use of approximators e.g., actor-critic networks to approximate the optimal policy and value function for continuous control tasks, where the feature representation capability greatly impacts learning efficiency and performance\cite{xu2013kernel,8986835}. Due to the requirement for online fast adaptation ability to dynamic environments of motion planning algorithms, it is essential to enhance the feature representation capability for improving the learning efficiency. In our work, we adopt lightweight basis functions to construct features of actor and critic for  online efficient learning since deep neural networks are usually with the deep RL framework where the control policy is trained offline and deployed online. Still, expanding the sample size for better feature representations may increase the dimensions of the basis function and lead to higher computational complexity. As far as we know, no prior RL-based planning approaches have addressed automatic feature construction for efficient and near-optimal online policy learning.
	
	Motivated by the above challenge, we propose a sparse kernel-based RL algorithm for online efficient and near-optimal motion planning with uncertain dynamics. Specifically, we utilize a quantified kernel sparsification technique\cite{xu2007kernel,engel2004kernel} to enhance the feature representation ability. In a unified manner, this kernel sparsification technique is manifested in basis functions for both the actor and critic in RL and for Gaussian processes (GPs), thereby improving the online adaptation capability (please refer to the results in Sec.~\ref{AdaptionCapability}). Our approach differs from optimization-based methods\cite{hewing2019cautious,hewing2018cautious,kabzan2019learning}, as it allows for offline training and rapid online deployment and can generalize well to similar scenarios. In contrast, optimization-based methods may face issues with computational efficiency and reliability when solving nonlinear optimization problems online.
	\section{Related Work}
	We discuss GP-based  techniques and present literature review on model-based RL approaches (MBRLs) for systems with uncertain dynamics and state constraints.
	
	\textbf{GP-based model identification}.  Rasmussen \cite{rasmussen2003gaussian} presented the idea of formulating GP models as a Bayesian framework for regression problems of stochastic processes. Sparse GPs were previously proposed to reduce the computational load by exploiting the structure of matrices. By selecting inducing points corresponding to the local trajectory, sparse GPs were incorporated into MPC algorithms to reduce the conservativeness\cite{hewing2019cautious,hewing2018cautious,kabzan2019learning}. It is assumed that the reference signals typically exhibit minimal variation between adjacent time steps.  Motivated by these sparse GP-based methods, we adopt a quantification approach of approximate linear dependence (ALD) \cite{xu2007kernel,engel2004kernel} based on correlation measurement to obtain less correlated samples. Leveraging this strategy, we integrate it into the RL planning algorithm to address the trade-off between efficiency and feature representation capability in solving nonlinear optimization problems.

	\textbf{MBRL-based optimal control for systems with uncertain dynamics.}  In \cite{deisenroth2011pilco}, PILCO propagated model uncertainties for a feedback policy by using a gradient-based policy search in planning horizons. It is data-efficient for learning control policies from scratch. In \cite{xie2016model,nagabandi2018neural}, MPC algorithms with neural network (NN) based system dynamics were designed to improve the efficiency of MBRL. Related work on improving the sample efficiency can also be found in \cite{zhou2020deep,luis2023model}. These approaches focus on providing training samples for MBRL algorithms to enhance training efficiency. However, we adopt a sparse kernel-based technique to enable the online adaptation of the RL approach, thereby directly improving the planning and control performance. Under the MBRL framework, Bechtle et al. \cite{bechtle2020curious} proposed \emph{Curious iLQR}, which combines iLQR and Bayesian modeling of uncertain dynamics. Extensive simulations and real-world experiments on 7-DoF manipulators validate its superiority. In contrast, our approach is a sparse kernel-based RL method trained with batch data samples and the updated model. Therefore, it allows for efficient training and rapid online deployment and can generalize well to similar scenarios. In this manner, it does not need to solve nonlinear optimization problems at each time step.
	
	\textbf{MBRL-based motion planning for systems with uncertain dynamics.} Kamthe et al. \cite{kamthe2018data} utilized probabilistic MPC to generate trajectories, and employed Pontryagin's Maximum Principle to generate safe control policies. In \cite{he2021integral}, the unknown disturbances are considered in the kinodynamic motion problem and solved by Pontryagin's Maximum Principle. Zhang et al. \cite{9756946} proposed a receding-horizon RL algorithm with an NN-based uncertain dynamics model for motion planning of intelligent vehicles, while it is generally difficult to collect a sufficient number of samples. In \cite{mahmud2021safety}, a barrier transformation was incorporated into the RL algorithm for generating optimal control inputs online, but there may exist issues of slow convergence speed and difficulty in selecting appropriate parameters.  Performance measurement-based approaches can be found in \cite{hoerger2020non,bejjani2018planning}. In \cite{hoerger2020non}, it focuses on deciding when to linearize the system dynamics for mitigating the effect of uncertain dynamics. Bejjani et al. proposed using a sampling-based planner to generate planning results for learning the optimal value function, which is then further optimized through RL. In summary, it remains challenging for these studies to ensure the efficiency and safety of online training, while our model-based batch RL with online adaption capability provides an effective and practical solution for vehicle kinodynamic motion planning.
	
	\textbf{Contribution.} 
	1) We propose a sparse kernel-based RL motion planning algorithm, called GP-SKRL, which updates its policies in a batch training mode. In a unified manner, GP-SKRL utilizes a sparse kernel-based quantification technique to estimate the uncertain dynamics and train the RL policies. This enhances online adaption capability and the convergence of the algorithm is mathematically proved. 2) The proposed RL planning algorithm can maintain near-optimal performance under uncertain vehicle dynamics and collision constraints. Extensive simulation and experimental results show that GP-SKRL outperforms several advanced optimization-based methods. 3) Due to the consideration of dynamic characteristics during the learning process, GP-SKRL can directly compute control sequences for intelligent vehicles. In particular, real-world experiments on a Hongqi E-HS3 electric vehicle validate its efficiency and effectiveness.
	\section{Preliminaries}
	\subsection{Problem Formulation}
	Consider the discrete-time vehicle dynamics in \cite{hewing2019cautious}:
		\begin{equation}\label{true_model}
			{x}_{k+1} =f_{\text{nom}}\left( {x}_k ,{u}_k \right) + B_d\underbrace{\left( g(x_k,u_k)+w_k \right)}_{y_k},
		\end{equation}
		where ${x}_k\in\mathbb{R}^{n_{x}}$ is the system state and $u_k\in\mathbb{R}^{n_u}$ is the control input. The model consists of a known nominal part $f_{\text{nom}}$, an additive term $g$, which lies within the subspace spanned by $B_d$\cite{hewing2019cautious}. We assume that the process noise $w_k \sim \mathcal{N}\left(0, \Sigma^w\right)$ is independent and identically distributed (i.i.d.), with spatially uncorrelated properties, i.e., $\Sigma^w=\text{diag}\{\sigma_1^2, \ldots, \sigma_{n_x}^2\}$. We also assume that both $f_{\text{nom}}$ and $g$ are differentiable. Here $y_k$ is the uncertain dynamics, which are to be learned from data.
	
	Consider the reference trajectory as follows:
		\begin{equation}\label{ref_system}
			x_{k+1,r}=\mathcal{P}_r(x_{k,r},u_{k,r}),
		\end{equation}
		where $x_{k,r}$ is the reference state, $u_{k,r}$ is the reference control input, and $\mathcal{P}_r(\cdot,\cdot)$ is a smooth mapping function. With Eqs.~\eqref{true_model} and~\eqref{ref_system}, one can obtain the error model (deferred in Appendix.~\ref{nominal_dynamics}) as
		\vspace{-1mm}
		\begin{equation*}\label{error_model}
			\mathbf{x}_{k+1} =f\left( {\mathbf{x}}_k ,\mathbf{u}_k \right),
		\end{equation*}
		where $\mathbf{x}$ is the error state and $\mathbf{u}$ is the control input.
	
	Given an initial state $x_0\in\mathbb{R}^{n_x}$, generate the real-time optimal control $\mathbf{u}^*_k$ at the $k$-th time instant.
	The infinite-horizon optimal motion-planning problem can be formulated as
	\begin{equation}\label{optimal_valuefunction}
		\begin{array}{c}\min_{\mathbf{u}_k}\end{array}V(\mathbf{x}_k,\mathbf{u}_k), 
	\end{equation}
	and satisfy the following conditions: 1) Starts at $x_0$ and tracks the reference trajectory $\mathcal{P}_r$. 2) Avoids collisions with all obstacles $\mathcal{B}_1,\dots, \mathcal{B}_q\subseteq\mathcal{W}$, where $\mathcal{W}\in\mathbb{R}^2$ denotes the Euclidean workspace. 
	We define the value function as the cumulative discounted sum of infinite-horizon costs:
	\begin{equation*}\label{value_function}
		\begin{array}{c}
			V\left(\mathbf{x}_k, \mathbf{u}_k\right)=\sum_{i=k}^{\infty}\gamma^{i-k} L\left(\mathbf{x}_i, \mathbf{u}_i\right),
		\end{array}
	\end{equation*}
	where $0<\gamma\le1$ and we define $L(\mathbf{x}_i, \mathbf{u}_i)=\mathbf{x}_i^{\top}Q\mathbf{x}_i+\mathbf{u}_i^{\top}R\mathbf{u}_i$ as the stage cost function, where $Q\in\mathbb{R}^{n_{x}\times{n_x}}$ is positive semi-definite and $R\in\mathbb{R}^{n_u\times n_u}$ is positive definite.
	
	The above optimal value function can be obtained by applying the optimal policy $\mathbf{u}^*$ to the system, which is computed by setting $\partial V/\partial\mathbf{u}=0$, i.e.,
		\begin{equation}\label{optimal_action}
			\begin{array}{cc}
				\mathbf{u}^{*}_k = -\frac{1}{2}\gamma R^{-1}({\partial {\mathbf{x}}_{k+1}}/{\partial \mathbf{u}_k})^{\top}\lambda^*(\mathbf{x}_{k+1}),
			\end{array}
		\end{equation}
		where the optimal costate $\lambda^*_{k+1}=\partial V^*(\mathbf{x}_{k+1})/\partial \mathbf{x}_{k+1}$.
		With $L(\mathbf{x}_k, \mathbf{u}_k)$ and~\eqref{optimal_valuefunction}, the optimal costate is computed by
		\begin{equation}\label{optimal_lambda}
			\begin{aligned}
				\lambda^{*}_k=2Q\mathbf{x}_k+\gamma({\partial{\mathbf{x}_{k+1}}}/{\partial \mathbf{x}_k})^{\top} \lambda^{*}(\mathbf{x}_{k+1}).
			\end{aligned}
		\end{equation}
\subsection{Approximate Linear Dependence (ALD)}
	In terms of selecting dictionaries online, readers can refer to \cite{kabzan2019learning}. Then the dictionaries $\mathcal{D}_{\rm{GP}}$ and $\mathcal{D}_{\rm{RL}}$ undergo the sparsification process by utilizing ALD. We briefly outline the steps of ALD below (please refer to \cite{xu2007kernel,engel2004kernel} for details):
	
	Step 1: Given each sample $z_t$ from a collected sample set, the following optimization problem is formulated: 
	\begin{equation}\label{FormulatedALDApproach}
		\begin{array}{c}
			\delta_{t}=\min _{c}\left\|{\sum_{j=1}^{{t-1}}} c_{j} \phi\left(z_{j}\right)-\phi\left(z_{t}\right)\right\|^{2},
		\end{array}
	\end{equation}
	where $c=\left[c_{1}, c_{2}, \cdots, c_{{t-1}}\right]^{\top}\in\mathbb{R}^{(t-1)\times n_z}$ constitutes the current sparse dictionary $\mathcal{D}_{t-1}$. According to the Mercer kernel theorem\cite{phienthrakul2005evolutionary}, there exists a mapping function $\phi(\cdot)$ that maps the state from the collected sample set to the Hilbert space $\mathcal{H}$. Since the Gaussian kernel function $k(\cdot,\cdot)$ satisfies the Mercer kernel condition, {one can obtain the following equality:}
	\begin{equation*}
		k(z_i,z_j)=\left<\phi(z_i),\phi(z_j)\right>,
	\end{equation*}
	where $\left< \cdot ,\cdot \right> $ denotes the inner product in the Hilbert space.
	The distance $\delta_{t}$ between $z_t$ and $\mathcal{D}_{t-1}$ can be computed by~\eqref{FormulatedALDApproach}.

	Step 2: The sparse dictionary $\mathcal{D}_t$ is updated  by comparing $\delta_{t}$ with a preset threshold $\delta_{\rm{max}}$. If $\delta_{t}>\delta_{\rm{max}}$, $\mathcal{D}_t=\mathcal{D}_{t-1}\cup z_t$; Otherwise, we discard it; i.e., $\mathcal{D}_t=\mathcal{D}_{t-1}$.
	
		The final kernel dictionaries $\mathcal{D}_{\rm{SGP}}$ and $\mathcal{D}_{\rm{SRL}}$ are used for sparse GP regression and sparse kernel-based RL learning processes, respectively.
\subsection{Sparse GP Regression}
Next, we will review a sparse GP regression method called FITC \cite{snelson2005sparse}, which reduces computational complexity by selecting inducing samples and introduces a low-rank approximation of the covariance matrix, transforming the original GP model into an efficient one. It is briefly introduced in the following.
\subsubsection{The formulation of full GP Regression}
An independent training set is composed of state vectors, i.e., $\mathbf{z}=[z_1,z_2,\cdots,z_n]^{\top}\in\mathbb{R}^{n\times n_z}$ 
	and the corresponding output vectors $\mathbf{y}=[y_1,y_2,\cdots,y_n]^{\top}\in\mathbb{R}^{n\times n_{y}}$. In \cite{rasmussen2003gaussian}, the mean and variance functions of each output dimension $a\in\{1,\cdots,n_{y}\}$ at a test point $z=[x^\top,u^\top]^\top$ are computed by
	\begin{equation}\label{mean_covariance}
		\begin{aligned}
			&m^{a}_{d}=K^{a}_{z\mathbf{z}}(K^{a}_{\mathbf{zz}}+I\sigma^2_{a})^{-1}[\mathbf{y}]_a\\
			&\Sigma^{a}_{d}=K^{a}_{zz} -K^{a}_{z\mathbf{z}} \left( K^{a}_{\mathbf{zz}} +\sigma _{a}^{2}I \right) ^{-1}K^{a}_{\mathbf{z}z},\\
		\end{aligned}
	\end{equation}
	where $\sigma_{a}$ is the variance, $K_{\mathbf{zz}}^a=k^a(\mathbf{z},\mathbf{z})\in\mathbb{R}^{n_z\times n_z}$ is a Gram matrix containing variances of the training samples. Correspondingly, $K_{z\mathbf{z}}^a=(K_{\mathbf{z}z}^a)^{\top}=k^a(z,\mathbf{z})$ denotes the variance between a test sample and training samples, and $K_{zz}^a =k^a(z,z)$ represents the  covariance,
	$k^{a}(\cdot,\cdot)$ is the kernel function and defined as follows:
	\begin{equation}
		k^a\left(z_i, z_j\right)=\sigma_{f, a}^2 \exp (-1/2\left(z_i-z_j\right)^{\top} L_a^{-1}\left(z_i-z_j\right)),	
	\end{equation}
	where $\sigma_{f, a}^2$ is the signal variance and $L_a=\ell^2 I$. Here $\sigma_{f,a}$ and $\ell$ are hyperparameters of the covariance function.
\subsubsection{Sparse GP Regression}
Given a dictionary set $\{\mathbf{z}_{\rm{ind}},\mathbf{y}_{\rm{ind}}\}$ with $n_{\rm{ind}}$ samples from $\{\mathbf{z},\mathbf{y}\}$, the prior hyper-parameters can be optimized by maximizing the marginal log-likelihood of the observed samples. In \cite{snelson2005sparse}, the mean and variance functions of a full GP are approximated by using inducing targets $\mathbf{y}_{\rm{ind}}$, inputs $\mathbf{z}_{\rm{ind}}$, i.e.,

\begin{equation}\label{FITC}
	\begin{aligned}
		&\tilde{m}_a^d(z)=Q_{z \mathbf{z}}^a(Q_{\mathbf{z} \mathbf{z}}^a+\Lambda)^{-1}[\mathbf{y}]_{a}, \\
		&\tilde{\Sigma}_a^d(z)=K_{z z}^a-Q_{z \mathbf{z}}^a(Q_{\mathbf{z} \mathbf{z}}^a+\Lambda)^{-1} Q_{\mathbf{z} z}^a,
		\end{aligned}
\end{equation}
where $\Lambda=\operatorname{diag}\{K_{\mathbf{z}\mathbf{z}}^a-Q_{\mathbf{z}\mathbf{z}}^a+I \sigma_a^2\}$ and the notation $ Q_{\zeta \tilde{\zeta}}^a:=K_{\zeta \mathbf{z}_{\rm{ind}}}^a(K_{\mathbf{z}_{\rm{ind}} \mathbf{z}_{\rm{ind}}}^a)^{-1} K_{\mathbf{z}_{\rm{ind}} \tilde{\zeta}}^a$. Several matrices in~\eqref{FITC} do not depend on $z$ and can be precomputed, such that they only need to be updated when updating $\mathbf{z}_{\rm{ind}}$ or $\mathcal{D}$ itself.

A multivariate GP is established by combining $n_y$ outputs, i.e.,
	\begin{equation}
		d(z) \sim \mathcal{N}(m_d,\Sigma_{d}),
	\end{equation}
	where $m_d=[m_d^1,\cdots,m_d^{n_y}]^{\top}$, and $\Sigma_{d}=\text{diag}\{\Sigma_{d}^1,\cdots,\Sigma_{d}^{n_y}\}$.

\begin{figure}
	\centering\includegraphics[width=3.5in]{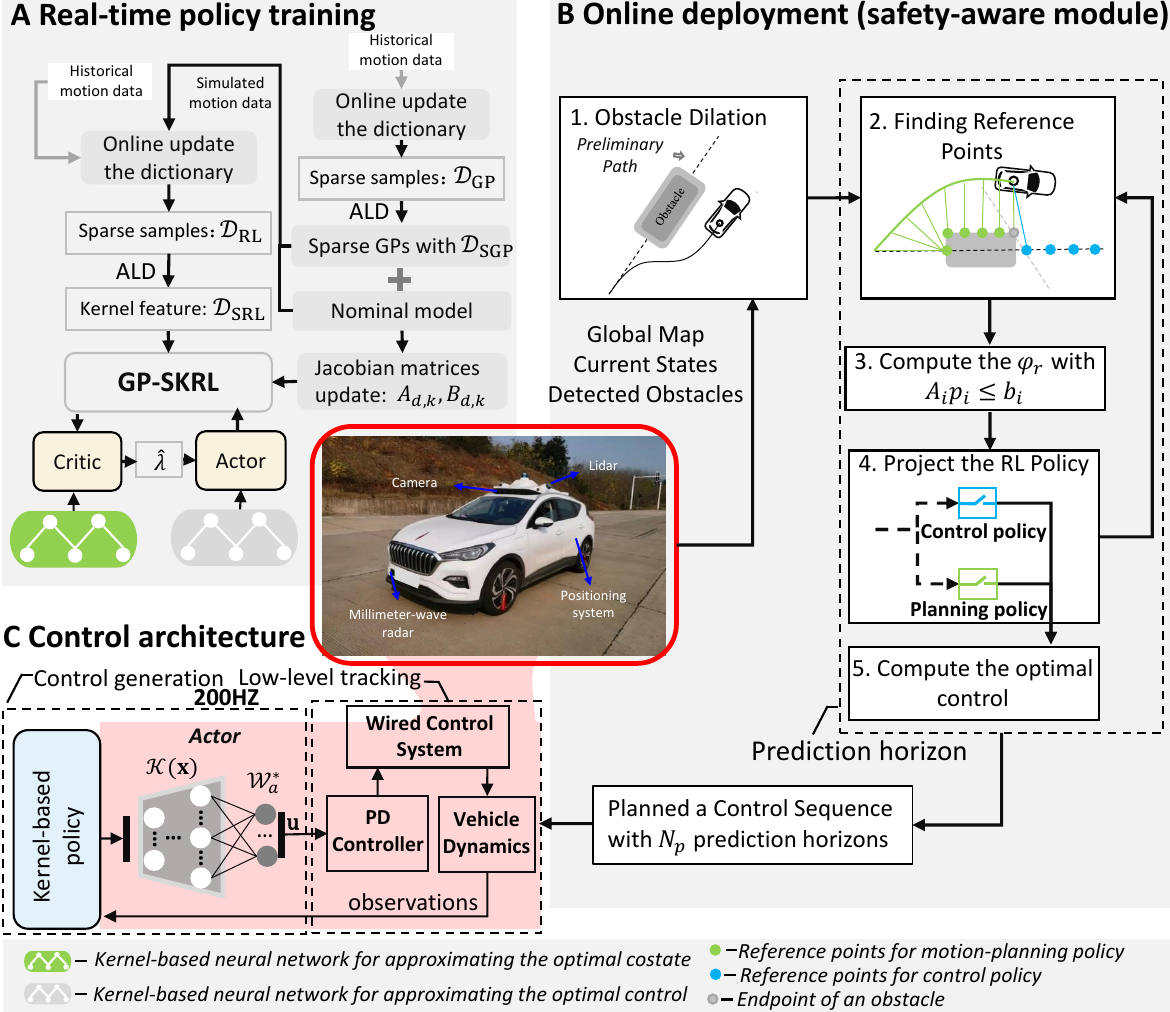}
	\caption{Illustration of {GP-SKRL}-based motion planning. }
	\label{fig_2}
\end{figure}

\section{Sparse Kernel-based Reinforcement Learning with Online Adaptation Capability}\label{Kernel-basedRL}
To obtain a near-optimal yet efficient performance, we propose a sparse kernel-based RL method. We employ a quantized sample sparse technique. Under a unifying manner, the sparsity is not only applied in RL but also in GP. This technique enables efficient and rapid adaptation from both model learning and policy learning perspectives. In this section, we will present the specific details of the sparse kernel-based RL algorithm. Firstly, we introduce the design of sparse GPs to facilitate the learning of model uncertainties. Building upon this foundation, we apply the sparse technique to design a kernel-based RL framework. Finally, we present a safety-aware module for deploying the RL policies to intelligent vehicles.
\subsection{Sparse GP Regression for Learning Model Uncertainties}\label{SGPRLU}
Nominal vehicle models often fail to perfectly represent the exact dynamics. Thus it is necessary to identify the differences to achieve better planning performance. 
To construct the ``input-output" form of a GP, Eq.~\eqref{true_model} is rewritten as
\begin{equation}\label{error_model}
	{d}(z_k) = B_d^{\dagger}({x}_{k+1} -f_{\text{nom}}\left({x}_k ,{u}_k \right)),
\end{equation}
where $B_d^{\dagger}$ is the Moore–Penrose pseudoinverse. Historic vehicle motion data is collected to capture the exact dynamics.

With the optimized parameters, GP models are utilized to compensate for uncertainties. Therefore, we define the learned model of~\eqref{true_model} as
\begin{equation}\label{Dynamics_for_planning}
	x_{k+1}=f_{\text{nom}}({x}_k, {u}_k)+{d}(z_k)+\epsilon,
\end{equation}
where $\epsilon\in\mathbb{R}^{n_x}$ is the estimation error. 
	\begin{remark}
		Given a training data set with $n$ samples, the computational complexity of the full GP is $\mathcal{O}(n^3)$, while it is $\mathcal{O}(nn^2_{\rm{ind}})$ for FITC. If we obtain a sparse dictionary $\mathcal{D}_{\rm{SGP}}$ with $n_{\rm{sp}} \ll n$ samples, the complexity is $\mathcal{O}(n_{\rm{sp}}n^2_{\rm{sp},\rm{ind}})$, where $n_{\rm{sp},\rm{ind}}$ is the inducing samples obtained from $\mathcal{D}_{\rm{SGP}}$.
\end{remark}
The Jacobian matrices of~\eqref{Dynamics_for_planning} are
	\begin{equation}\label{JocabianMatrices}
		\begin{array}{l}
			A_{d,k}=A_{\text{nom},k}+\frac{\partial {d}(z_k)}{\partial {{x}_k}}, B_{d,k}=B_{\text{nom},k}+\frac{\partial {d}(z_k)}{\partial {{u}_k}},
		\end{array}
	\end{equation}
	where $A_{d,k}\in\mathbb{R}^{n_x\times n_x}$, $B_{d,k}\in\mathbb{R}^{n_x\times n_u}$, $A_{\text{nom},k}=I+{T}_s\partial f_{\text{nom}}^0/\partial x$ and $B_{\text{nom},k}={T}_s\partial f_{\text{nom}}^0/\partial u$, ${T}_s$ is the sampling interval and $f_{\text{nom}}^0$ is the nominal dynamics. 
	Finally, the following model is used for training:
	\begin{equation}\label{identification_form_dynmaic}
		\mathbf{x}_{k+1}=A_{d,k}\mathbf{x}_k + B_{d,k}\mathbf{u}_k,
	\end{equation}
	where the error state $\mathbf{x}$ is defined by $x-x_r$; Namely, $\mathbf{x} = [e_{v_x},e_{v_y},e_{\varphi},e_{\omega},e_X,e_Y]^{\top}$, and $\mathbf{u}=[a_x, \delta_{f}]^{\top}$.
	\subsection{Sparse Kernel-based RL with State Constraints}\label{4.2}
	
\subsubsection{Cost function reformulation with barrier function}
For the vehicle dynamics in Sec.~\ref{nominal_dynamics}, we define a projection function $\psi:\mathbb{R}^6\rightarrow\mathbb{R}^2$  by $(v_x, v_y, \varphi, \omega, X, Y) \mapsto (e_X,e_Y)$, mapping the current vehicle states to the position errors.
The safety constraint is incorporated into the barrier function, i.e.,
\begin{equation}\label{barrierfunction}
	\mathcal{B}({x}_k) = \exp \left( -\lVert \psi({x}_k)  \rVert \right).
\end{equation}
Thus the cost function $L$ is re-defined as follows:
\begin{equation}\label{motion_planning_object}
	\begin{aligned}
		L(\mathbf{x}_k, \mathbf{u}_k)	= \underbrace{\mathbf{x}_k^{\top}Q\mathbf{x}_k}_{\emph{state term}} + \underbrace{\mathbf{u}_k^{\top}R\mathbf{u}_k}_{\emph{control term}}+\underbrace{\mu\mathcal{B}({x}_k)}_{\emph{barrier term}},
	\end{aligned}
\end{equation}
where the penalty coefficient $\mu$ is a non-negative constant, $Q=\mathrm{diag}\{Q_1, 0, Q_3, 0, Q_5, Q_6\}\in\mathbb{R}^{6\times6}$ is positive semi-definite, and $R=\mathrm{diag}\{R_1, R_2\}\in\mathbb{R}^{2\times2}$ is positive definite.
\begin{figure}[ht]
	\centering
	\centering\includegraphics[width=2.5in]{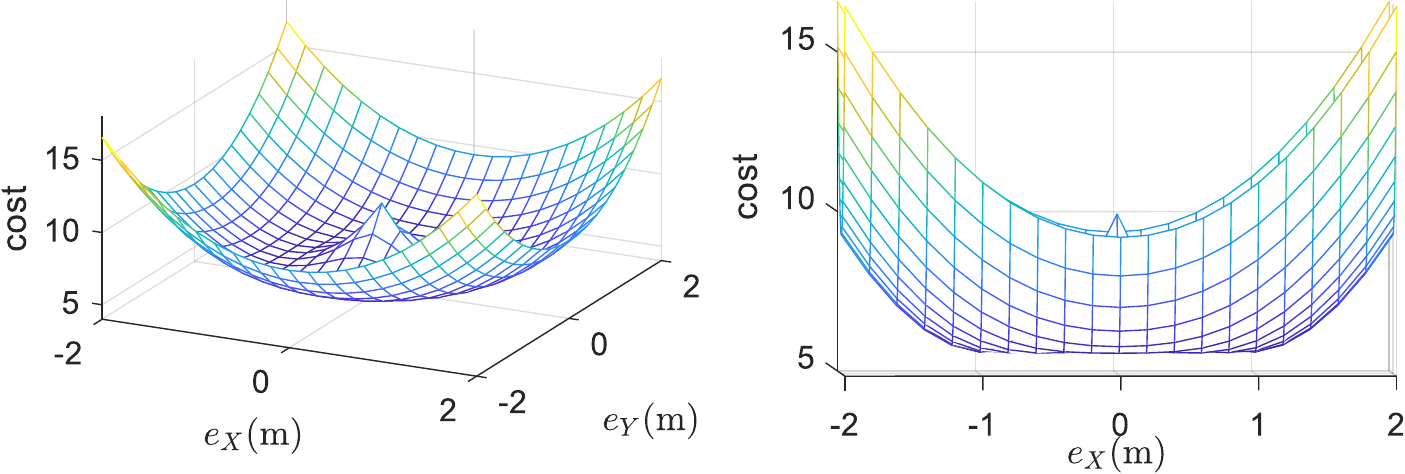}
	\caption{An illustration of the values of $L$ with $e_X$ and $e_Y$.}
	\label{fig_barrier}
\end{figure}
\begin{remark}
	Fig.~\ref{fig_barrier} illustrates the designed cost function $L$, which is a trade-off between training accuracy and safety. The \mt{state term} implies that a smaller value of $\lVert\psi({x_k})\rVert$ is preferred. As obstacles on the global path would lead to safety issues, the \mt{barrier term} comes into effect, driving the vehicle away from the \mt{desired path}. However, as $\lVert\psi({x_k})\rVert$ increases, the cost of the \mt{state term} becomes larger, which lowers tracking accuracy. Finally, one can obtain a policy considering both tracking accuracy and safety embedded in $L$.
\end{remark}
\subsubsection{Sparse kernel-based structure for approximating the optimal control and value function}
ALD can also be utilized to sparsify samples and hence acquire representative ones. Sparse kernel-based functions are adopted in actor-critic networks to construct a basis function vector for approximating $\lambda^*$ and $\mathbf{u}^*$ (specific derivation can be found in \cite{liu2013policy}). To this end, their estimations are	expressed using the sparse kernel function, i.e.,
\begin{subequations}
	\begin{equation}\label{hat_u} 
		\hat{\mathbf{u}}_k=\mathcal{W}_a^{\top}\mathcal{K}(\mathbf{x}_k),
	\end{equation}
	\begin{equation}\label{hat_lambda}
		\hat{\lambda}_k=\mathcal{W}_c^{\top}\mathcal{K}(\mathbf{x}_k),
	\end{equation}
	where $\mathcal{W}_a\in\mathbb{R}^{n_{\mathcal{K}}\times 2}$ and $\mathcal{W}_c\in\mathbb{R}^{n_{\mathcal{K}}\times 6}$ are the weights of the actor and critic networks, respectively, and $n_{\mathcal{K}}$ denotes the dimension of $\mathcal{D}_{\rm{SRL}}$. The basis function of state $\mathbf{x}_k$ is
		\begin{equation}\label{MultiKernelFeature}
			\mathcal{K}(\mathbf{x}_k)=\left[k(\mathbf{x}_k, c_{1}), \ldots, k(\mathbf{x}_k, c_{n_{\mathcal{K}}})\right]^{\top},
		\end{equation}
		where $c_1,\cdots, c_{n_{\mathcal{K}}}\in\mathbb{R}^6$ are the elements in $\mathcal{D}_{\rm{SRL}}$.
\end{subequations}
\subsubsection{Iterative batch-mode training framework using sparse kernel feature}\label{SKRL}
Here, we utilize a batch-mode RL learning strategy to train the actor-critic networks with a sample set $\mathcal{D}_{\rm{RL}}=\{\mathbf{x}_k\}_{k=1}^M$.
In terms of the actor network, the objective is to learn the optimal policy $\mathbf{u}^*$. At each iteration, the mean square cost to be minimized can be expressed as follows:
\begin{subequations}
	\begin{equation}\label{Ea}
		\begin{array}{cc}
			E_{a}=\lVert\hat{\mathbf{U}}-\mathbf{U}\rVert^{2}+\rho_a\lVert\mathcal{W}_{a}\rVert^2,
		\end{array}
	\end{equation}
	where $\hat{\mathbf{U}}=[\hat{\mathbf{u}}_1,\cdots,\hat{\mathbf{u}}_M]$ and $\mathbf{U}=[\mathbf{u}_1,\cdots,\mathbf{u}_M]$, $M$ is the number of samples, $\rho_a$ is a small positive number, and $\rho_a\lVert\mathcal{W}_{a}\rVert^2$ is an $L_2$ regularization term to limit the excessive increase of the weights during the update process. Here, $\mathbf{u}_{k}$ is the target control policy, which can be obtained by letting $\partial V/\partial \mathbf{u}=0$ and defined by
	\begin{equation}\label{target_action}
		\begin{array}{cc}
			\mathbf{u}_k=-{1}/{2}\gamma R^{-1}B_{d,k}^{\top}\hat{\lambda}_{k+1}.
		\end{array}
	\end{equation}
\end{subequations}
In the critic network, the target is to learn the optimal value function $\mathbf{\Lambda}$ of the planning task, which yields the following cost function to be minimized:
\begin{subequations}
	\begin{equation}\label{Ec}
		\begin{array}{cc}
			E_{c}=\Vert\hat{\mathbf{\Lambda}}-\mathbf{\Lambda}\Vert^{2}+\rho_c\lVert\mathcal{W}_{c}\rVert^2,
		\end{array}
	\end{equation}
	where $\hat{\mathbf{\Lambda}}=[\hat{\lambda}_1,\cdots,\hat{\lambda}_M]$ and $\mathbf{\Lambda}=[{\lambda}_1,\cdots,{\lambda}_M]$, $\rho_c\lVert\mathcal{W}_{c}\rVert^2$ is an $L_2$ regularization term,  $\rho_{c}$ is a small positive real number. Here $\lambda_k$ the target value function defined by
	\begin{equation}\label{target_lambda}
		\lambda_k=2Q\mathbf{x}_k+\mu\frac{\partial{\mathcal{B}({x}_k)}}{\partial{\mathbf{x}_k}}+\gamma A_{d,k}^{\top}\hat\lambda_{k+1}
	\end{equation}
\end{subequations}
Letting $\partial E_{a}/\partial \mathcal{W}_{a}$ and $\partial E_{c}/\partial \mathcal{W}_{c}$ be $0$, the update rules of actor-critic networks can be obtained.
Hence, at the $i$-th iteration, the weight update of the critic network is
\begin{equation}\label{WeightupdateWc}
	\mathcal{W}_{c} = (\mathbf{K}\mathbf{K}^{\top} +\rho _cI)^{-1}\mathbf{K}{\mathbf{\Lambda}}^\top,
\end{equation}
where $\mathbf{K}=[\mathcal{K}(\mathbf{x}_1),\cdots,\mathcal{K}(\mathbf{x}_M)]$.
In the same way, the weight update of the actor network is
	\begin{equation}\label{WeightupdateWa}
		\mathcal{W}_{a} = (\mathbf{K}\mathbf{K}^{\top} +\rho _aI)^{-1}\mathbf{K}{\mathbf{U}}^\top.\\
\end{equation} 

{GP-SKRL} is trained for a combinatorial near-optimal planning policy based on the collected samples. After the weights of the actor-critic networks converge at the $i$-th iteration, i.e., $\lVert \mathcal{W}_{a,i}- \mathcal{W}_{a,i-1} \rVert^{2}\le\sigma _a$ and $\lVert \mathcal{W} _{c,i} - \mathcal{W} _{c,i-1} \rVert^{2}\le\sigma_c$, where $\sigma_a$ and $\sigma_c$ are small positive real numbers, the learned policies can be deployed to real-time motion planning tasks. 

\begin{algorithm}[t]
	\footnotesize
	\caption{{GP-SKRL} algorithm}
	\label{algorithm_KACMP}
	\LinesNumbered
	\SetKwComment{Comment}{//}{}
	\KwIn{Hyper-parameters $\gamma, \rho_a, \rho_c, \sigma_a, \sigma_c$,  randomly initialized $\mathcal{W}_a, \mathcal{W}_c$, pre-collected datasets $\mathcal{D}_{\rm{GP}}, \mathcal{D}_{\rm{RL}}$.} 
	
	\KwOut{The converged policy $\pi$.}
	
	
	
	Update the dictionaries $\mathcal{D}_{\rm{GP}}$ and $\mathcal{D}_{\rm{RL}}$ online.
	
	Obtain the sparse dictionaries $\mathcal{D}_{\rm{SGP}}$ and $\mathcal{D}_{\rm{SRL}}$ using ALD.
	
	Train GP models using $\mathcal{D}_{\rm{GP}}$ and $\mathcal{D}_{\rm{SGP}}$ (inducing points).

	\For{$\varrho \in \{0, 1\}(control~and~planning)$}{
		$\varrho=0: \mu=0$; \tcp{\textit{Control policy}}
		
		$\varrho=1: \mu\gets \text{a positive real number}$. \tcp{\textit{Planning policy}}
		
		
		\While{policy has not converged
		}{
			
				\For{$k=1 \cdots M\ (sample~number~of~\mathcal{D}_{\rm{RL}})$}{
					Update $A_{d,k},B_{d,k}$ using GP models yields Eq.~\eqref{identification_form_dynmaic}.
					
					Obtain $\hat{\lambda}_k$ and $\hat{\mathbf{u}}_k$ with~\eqref{hat_lambda} and~\eqref{hat_u}.  
					
					Construct kernel feature $\mathcal{K}(\mathbf{x}_k)$ with~\eqref{MultiKernelFeature}.
					
				}
				
				\For{$k=1 \cdots M\ (sample~number~of~\mathcal{D}_{\rm{RL}})$}{
					
					Obtain ${\mathbf{x}}_{k+1}$ and ${\lambda}_{k+1}$ with~\eqref{identification_form_dynmaic} and~\eqref{hat_lambda}..
					

					Compute $\mathbf{u}_k$ and $\lambda_k$ with~\eqref{target_action} and~\eqref{target_lambda}.
				}
				Update actor-critic weights with Eqs.~\eqref{WeightupdateWc} and~\eqref{WeightupdateWa}.
				
			}
			$\pi(\varrho)=\mathcal{W}_{a}^*$.
		}
		
		
		\While{the vehicle has not reached the destination
		}{
			
			Compute the control input $\hat{\mathbf{u}}$ with~\eqref{hat_u} and perform steps $1$-$20$.
			
			\If{the planning and control policies converge}{
				Update the current policy $\pi$.}
			
		}
	\end{algorithm}
	
	\begin{remark}\label{computaitonal_complexity}
		The computational complexity of {GP-SKRL} in training approximates to $\mathcal{O}(M{n_{\mathcal{K}}^2})$. When the trained policy is applied to real-time tasks, the computational complexity is approximately $\mathcal{O}(2N_u{n_{\mathcal{K}}})$, where $N_u$ is the dimension of $u$. 
	\end{remark}

	\begin{remark}
		{Our approach is built upon kernel-based RL methods\cite{xu2013kernel,8986835}. On the one hand, we extend the previous kernel-based RL works to the online adaptive learning scenarios under the uncertain dynamics constraint. On the other hand, we incorporate cost-shaping into sparse kernel-based RL to realize near-optimal motion planning of intelligent vehicles.}
	\end{remark}
	
	 {In the following, we present the convergence analysis of the GP-SKRL algorithm.}
	\begin{thm}\label{thm1}
		(Convergence of {GP-SKRL}). As $i\rightarrow\infty$, sequences $\hat{\mathbf{\Lambda}}^{[i]}$ and $\hat{\mathbf{U}}^{[i]}$ will converge to $\mathbf{\Lambda}^*$ and $\mathbf{U}^*$, respectively, i.e., as $i\rightarrow\infty$, $\hat{\mathbf{\Lambda}}^{[i]}\rightarrow\mathbf{\Lambda}^*$ and $\hat{\mathbf{U}}^{[i]}\rightarrow \mathbf{U}^*$.
	\end{thm}
	\begin{proof}
		{See Appendix~\ref{AppendixProof}.}
	\end{proof}

	\subsection{Safety-Aware Module for Deploying RL Policies}\label{Method_safety_aware_module}
The safety-aware scheme (right half of Fig.~\ref{fig_2}) consists of five modules. The first module, obstacle dilation, dilates obstacles to ensure that the physical constraints of the vehicle are not violated. The second module, reference point determination, determines the current reference point by finding the nearest point on the \emph{desired path}, which is computed by the following strategy. When the vehicle is within the potential obstacle zone, the \mt{desired path} is adjusted to align with the contour of the obstacle. The control policy is employed to repeatedly calculate $N_p$-step real-time trajectories with the global path. If the trajectories collide with obstacles, the current \mt{desired path} is maintained; Otherwise, the \mt{desired path} is switched to the global path. Note that the global path can be obtained by connecting the starting and end points or through path planning techniques. For clarity, we provide an illustrative diagram depicting the variations of reference points during the execution of motion planning tasks in Fig.~\ref{fig_SwitchingScheme}.

\begin{figure}[ht]
	\centering
	\centering\includegraphics[width=2.5in]{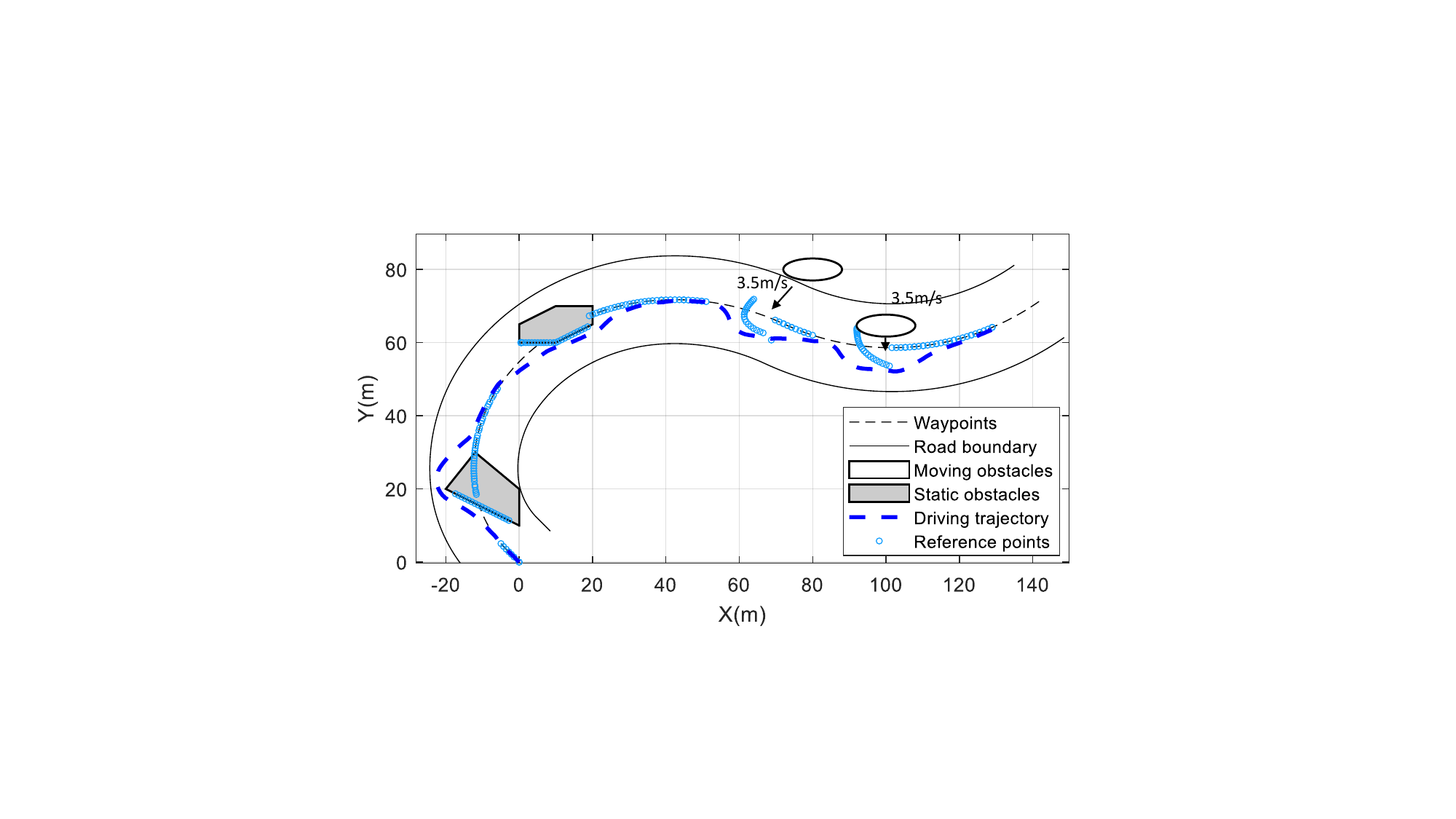}
	\caption{An illustration of selecting reference points in a planning task.}
	\label{fig_SwitchingScheme}
\end{figure}

When dealing with polygonal obstacles, the third module, reference heading angle calculation, employs a consistent strategy described in \cite{liu2018convex}. The established constraint at the $i$-th time instant is $A_i^{\top} p_i \leq b_i$, where $A_i=[a_{i,1}, a_{i, 2}]^\top\in\mathbb{R}^2$ and $-a_{i, 1}/a_{i, 2}$ is the slope of a straight line, $b_i\in\mathbb{R}$ is the intercept, and $p_i=(X_i,Y_i)^{\top}\in\mathbb{R}^2$ is the position of the vehicle. The reference heading angle is obtained using the inverse tangent function: $\varphi_{r,i} = \text{atan}(-a_{i, 1}/a_{i, 2})$. In the fourth module, RL policy projection, a switching scheme is implemented. The planning policy, denoted as $\pi(1)$, is applied when the \mt{desired path} is the obstacle's contour. Conversely, when the \mt{desired path} is the global path, the policy is adjusted to $\pi(0)$. The final module, optimal control computation, allows for the computation of optimal control using Eq.~\eqref{hat_u}.
	 
\section{Performance Evaluation and Comparisons}\label{Comparision}
The performance of {GP-SKRL} was tested through comparisons with kinodynamic RRT$^\star$ (Kd-RRT$^\star$)\cite{webb2013kinodynamic}, model predictive control with control barrier function (MPC-CBF)\cite{zeng2021safety}, local model predictive contouring control (LMPCC)\cite{brito2019model} and convex feasible set algorithm (CFS)\cite{liu2018convex} under specific metrics. Their characteristics and differences are shown in TABLE~\ref{tab:1}. Then we compared the tracking-control performance of {GP-SKRL} with and without model learning, abbreviated as {GP-SKRL} (w/ ML) and {GP-SKRL} (w/o ML), to validate the capability of learning uncertainties.
Furthermore, real-world experiments on a Hongqi E-HS3 electric car were performed.
\begin{table}[htbp]
	\centering\setlength\tabcolsep{1.5pt}
	\renewcommand\arraystretch{1.0}
	\tiny
	\scriptsize
	\caption{Comparison Between Methods Involved in Simulation or Real-world Experiments.}
	\label{tab:1} 
	\begin{threeparttable} 
		\begin{tabular}{cccccc}
			\toprule
			{{Methods}}  & {Kd-RRT$^\star$} & {MPC-CBF}& {LMPCC}&{CFS}& {GP-SKRL} \\ \midrule
			{Vehicle dynamics}&{\checkmark}&{\checkmark}&{\checkmark}&{\checkmark}&{\checkmark}\\
			{Online optimization}&$\times$&{\checkmark}&{\checkmark}&{\checkmark}&{\checkmark}\\
			{Offline training}&{\checkmark}&{$\times$}&{$\times$}&{$\times$}&\mc{\checkmark}\\
			{Dynamic obstacles}&$\times$&{\checkmark}&{\checkmark}&$\times$&{\checkmark}\\
			{Model online learning}&{$\times$}&{$\times$}&{$\times$}&{$\times$}&{\checkmark}\\\bottomrule
		\end{tabular}
	\end{threeparttable}
\end{table}
	
	\subsection{Performance Comparisons with Other Approaches}
	The settings and hyperparameters of the comparative methods discussed in this section are described in Appendix~\ref{Appendix_Settings}.
	\subsubsection{Evaluations of the computational results}
To better evaluate the {GP-SKRL} approach, we compare it with other kinodynamic motion planning algorithms under several specific metrics, which are Aver. S.T. (average solution time of each time step), $J_{\text{Lat}}$ (cost of the lateral error), $J_{\text{Lon}}$ (cost of the longitudinal error), $J_{\text{Heading}}$ (cost of the heading error), $J_{\text{Con}}$ (control cost), and $J_{\text{Safe}}$ (safety cost).
In this regard, we use a weighted average over all costs to evaluate the performance, which is expressed as follows:
	\begin{equation}\label{CostFunction}
	\footnotesize
	\begin{aligned}
		J = \begin{array}{c}\frac{1}{N}\sum_{k=1}^N (\mathcal{Q}_1J_{k,\text{Lon}}+\mathcal{Q}_2J_{k,\text{Lat}}+\mathcal{Q}_3J_{k,\text{Heading}}+J_{k,\text{Con}}),\end{array}
	\end{aligned}
\end{equation}
where $J_{\text{Lon}}=\lVert e_x \rVert^2, J_{\text{Lat}}=\lVert e_y \rVert^2, J_{\text{Heading}}=\lVert e_{\varphi} \rVert^2$, the control penalty matrix is  $\mathcal{R}=\text{diag}\{\mathcal{R}_1,\mathcal{R}_2\}\in\mathbb{R}^{2\times 2}$, where $\mathcal{R}_1, \mathcal{R}_2>0$, the control cost is $J_{\text{Con}}=\lVert \mathbf{u} \rVert_{\mathcal{R}}^2$, and $N$ denotes the number of the overall time steps. Another quantitative metric of overall length can be computed by 
\begin{equation*}
	\begin{array}{c}
		\text{Length} = \sum_{k=1}^{N-1} \lVert (X_{k+1},Y_{k+1})-(X_{k},Y_{k})\rVert.
	\end{array}
\end{equation*}

Additionally, when a planning algorithm is deployed on the vehicle, the completion time from the starting point to the goal point is denoted as $CT$.

\subsubsection{Analyses of simulation results in the first scenario}
We set the starting point as $(X,Y)=(5,58)$ and the end point at $(X,Y)=(238,50)$. A reference trajectory is obtained by connecting the two points with a constant reference heading angle of $-0.0343$ \rm{rad}. In this scenario, we set the maximum desired speed $V_{\text{max}}$ as $10$ \rm{m/s} for the testing approaches, and vehicles are required to avoid static polygonal obstacles. 
\begin{figure}[!htb]
	\centering\includegraphics[width=3.0in]{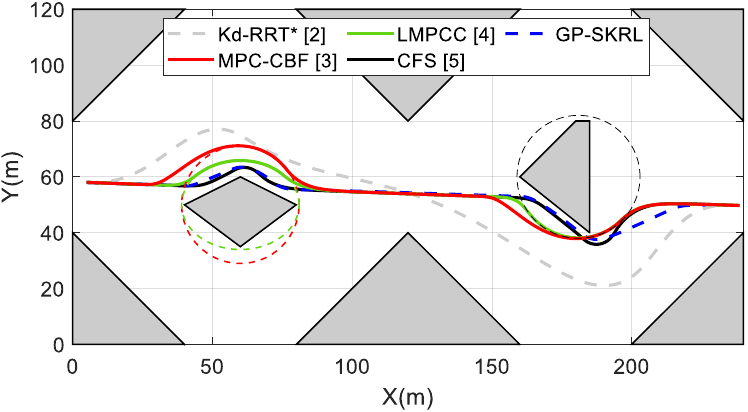}
	\caption{Obstacle avoidance in an unstructured road: the grey regions are static obstacles that are physically inaccessible. } 
\label{fig:CompareWithOthers}
\end{figure}

\begin{table}[!htb]\centering
\setlength\tabcolsep{1.5pt}
\renewcommand\arraystretch{1.0}
\scriptsize
\caption{Performance Evaluations of Fig.~\ref{fig:CompareWithOthers}}\label{table:MetricsComparisonI}
\begin{threeparttable}
	\begin{tabular}{ccccccccc}
		\toprule 
		{Quantitative Metrics}&{{GP-SKRL}}&{MPC-CBF}&{LMPCC}&{CFS}&Kd-RRT$^\star$\tnote{1}\\ \midrule 
		$J$&$\mathbf{44.91}$&{$105.90$}&{$60.86$}&$47.80$&$463.59$\\
		\text{Length}&$\mathbf{243.3}$ \rm{m}&{$249.3$} \rm{m}&{$246.3$} \rm{m}&$246.0$ \rm{m}&$261.6$ \rm{m}\\
		$CT$&$\mathbf{24.3}$ \rm{s}&{$25.8$} \rm{s}&{$24.8$} \rm{s}&$24.9$ \rm{s}&$25.62$ \rm{s}\\
		\text{Aver. S.T.}&$\mathbf{0.005}$ \rm{ms}&$100$ \rm{ms}&$80$ \rm{ms}&$60$ \rm{ms}&$174$ \rm{ms}\\ \bottomrule
	\end{tabular}
	\begin{tablenotes}
		\footnotesize
		\item[1] : Kd-RRT$^\star$ approach converges at the $236$-th iteration. 
	\end{tablenotes}
\end{threeparttable}
\end{table}

In Fig.~\ref{fig:CompareWithOthers}, the red dashed circle and green dashed ellipse are used to enclose the first static obstacle to establish the MPC-CBF and LMPCC constraints, respectively. The black dash circle is used to enclose the second static obstacle to establish the MPC-CBF and LMPCC constraints. {According to the testing results in our implementation, decreasing $\mathcal{Q}_1$ and $\mathcal{Q}_3$ while increasing $\mathcal{Q}_2$ enables better performance. Finally, state penalty values in~\eqref{CostFunction} are set as $\mathcal{Q}_1=\mathcal{Q}_2=2$ and $\mathcal{Q}_3=5$. The control penalty matrix $\mathcal{R}$ is $\text{diag}\{3,3\}$.} The comparison results are listed in TABLE~\ref{table:MetricsComparisonI}, which show that {GP-SKRL} makes the vehicle move along the obstacle boundary and achieves a better performance than LMPCC, MPC-CBF, and CFS in terms of  computation time, cost, and trajectory length. Under the constraint of $V_{\text{max}}$, {GP-SKRL} completes the task with the shortest time among the approaches.  It should be noted that {GP-SKRL} is more computationally efficient among all compared approaches due to the scheme of batch RL. 

\begin{figure}[!htb]
\centering
\centering\includegraphics[width=3.0in]{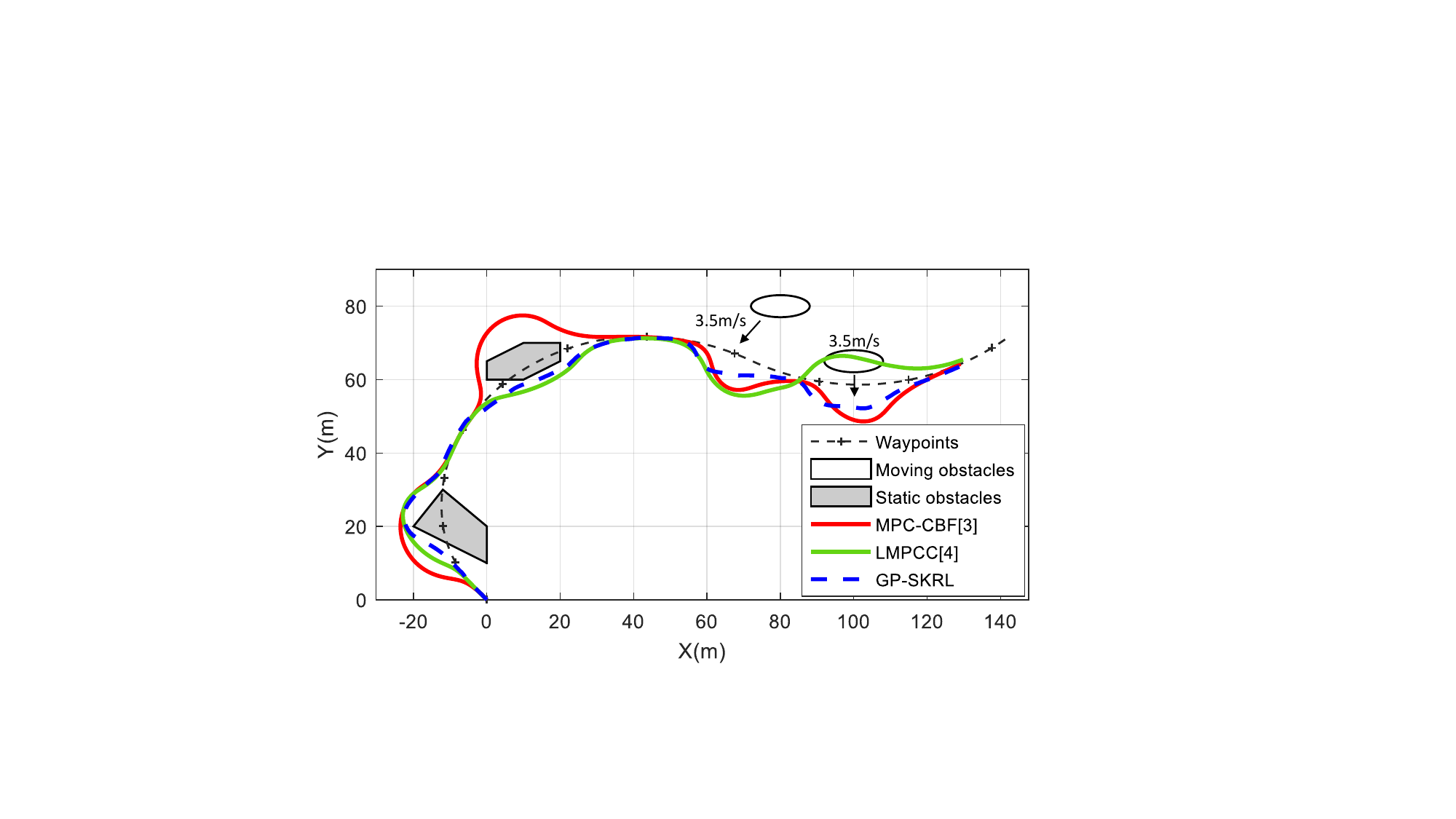}
\caption{Planning results in Scenario II with static and moving obstacles. The starting point coordinate is set as $(0,0)$. The moving obstacles (dark ellipses) are at their initial positions and the black arrows are their moving directions.}
\label{Fig_ScenarioII}
\end{figure}
\begin{table}[!htb]\centering
	\setlength\tabcolsep{8pt}
	\scriptsize
	\caption{Performance Evaluations of Fig.~\ref{Fig_ScenarioII}}\label{table:MetricsComparisonII}
	\begin{threeparttable}
		\begin{tabular}{ccccccccc}
			\toprule 
			Quantitative Metrics&{{GP-SKRL}}&{MPC-CBF}&{LMPCC}&CFS\\ \midrule
			$J$&$\mathbf{33.64}$&$94.28$&$54.62$&-\tnote{1}\\ 
			\text{Length}&$\mathbf{215.12}$ \rm{m}&$241.53$ \rm{m}&$217.41$ \rm{m}&-\\
			$CT$&$\mathbf{21.50}$ \rm{s}&$24.50$ \rm{s}&$22.0$ \rm{s}&-\\
			Aver. S.T. &$\mathbf{0.005}$ \rm{ms}&$90$ \rm{ms}&$100$ \rm{ms}&-\\
			\bottomrule 
		\end{tabular}
		\begin{tablenotes}
			\footnotesize
			\item[1] CFS with the kinodynamic constraint fails to avoid moving obstacles. 
		\end{tablenotes}
	\end{threeparttable}
\end{table}
\subsubsection{Analyses of simulation results in the second scenario}
To implement the comparison methods, we enclose all obstacles separately with circles and ellipses, and the black line with cross marks is the reference trajectory (see Fig.~\ref{Fig_ScenarioII}). In this scenario, the maximum desired speed is set as $10$ \rm{m/s} for the testing approaches. As shown in Fig.~\ref{Fig_ScenarioII}, MPC-CBF, LMPCC, and {GP-SKRL} successfully enable the vehicle to avoid obstacles and arrive at the destination. In the following, we analyze the specific metrics of them in completing this task.

As shown in TABLE~\ref{table:MetricsComparisonII}, {GP-SKRL} is more advantageous than MPC-CBF and LMPCC regarding computation time, average cost and route length. Regarding cost, MPC-CBF and LMPCC use circles or ellipses to enclose all obstacles, which results in too many extra safe areas being included. Likewise, {GP-SKRL} can obtain a shorter route among the comparison approaches. And it takes a shorter computational time. 
\subsubsection{Validation of the model learning scheme}
To improve planning and control performance, {GP-SKRL} trains GP models to learn model uncertainties. The parameters of the nominal model are set as $m=20,000$ \rm{kg} and $I_z=20,000$ \rm{kg}$\cdot$\rm{m}, while $m=2257$ \rm{kg} and $I_z=3524.9$ \rm{kg}$\cdot$\rm{m} are set for the exact model.  Hyperparameters are set as $\sigma_{f, a}=5$ and $\ell=2$. {Due to the structure of the nominal model, parameter uncertainties are assumed to only affect $v_y$ and $\omega$ of the system; that is, $g(x, u)=g\left(v_x, v_y, \omega, a_x, \delta_f\right): \mathbb{R}^5 \rightarrow \mathbb{R}^2$. The i.i.d. process noise $w_k\in\mathbb{R}^6$ is added to the vehicle dynamics, where $\Sigma^w=0.001\cdot\rm{diag}\{1/3,\cdots,1/3\}$. The proposed sparse GP is employed to identify model uncertainties on all six states.}

\begin{figure}[ht]
	\centering\includegraphics[width=2.0in]{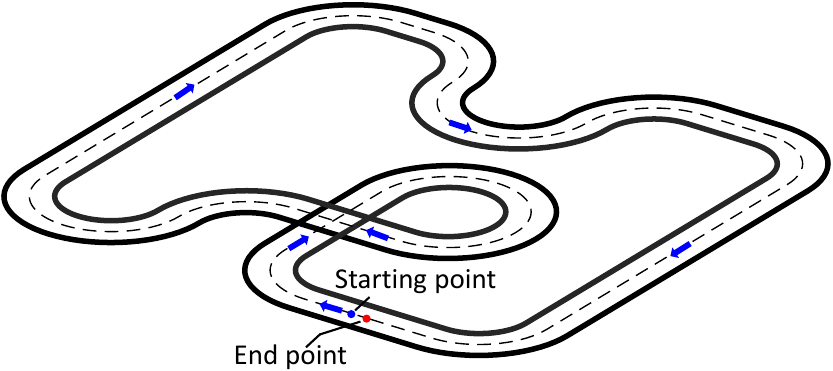}
	\caption{Racing road with a desired speed of ${10}$ \rm{m/s} and maximum length of $508$ \rm{m}. The black dot curve is the reference path and the black solid lines are the road boundaries. We use blue arrows to represent the driving directions.}
	\label{fig:CompModelLearning}
\end{figure}

After collecting the training samples, the posterior hyper-parameters were inferred from the marginal log-likelihood optimization. Fig.~\ref{fig:CompModelLearning} depicts the racing road for  testing the tracking-control performance of {GP-SKRL} (w/ ML) and {GP-SKRL} (w/o ML). We counted the lateral stage errors and the average lateral errors of their two driving trajectories to present the performance improvement clearly. In Fig.~\ref{fig:Error_ComparisonTrackingWithAndWithoutModelLearning}, using the ML strategy reduces the lateral tracking control error, demonstrating that the ML strategy improves the planning and control performance of {GP-SKRL} under model uncertainties.

\begin{figure}[!ht]
\centering\includegraphics[width=3.0in]{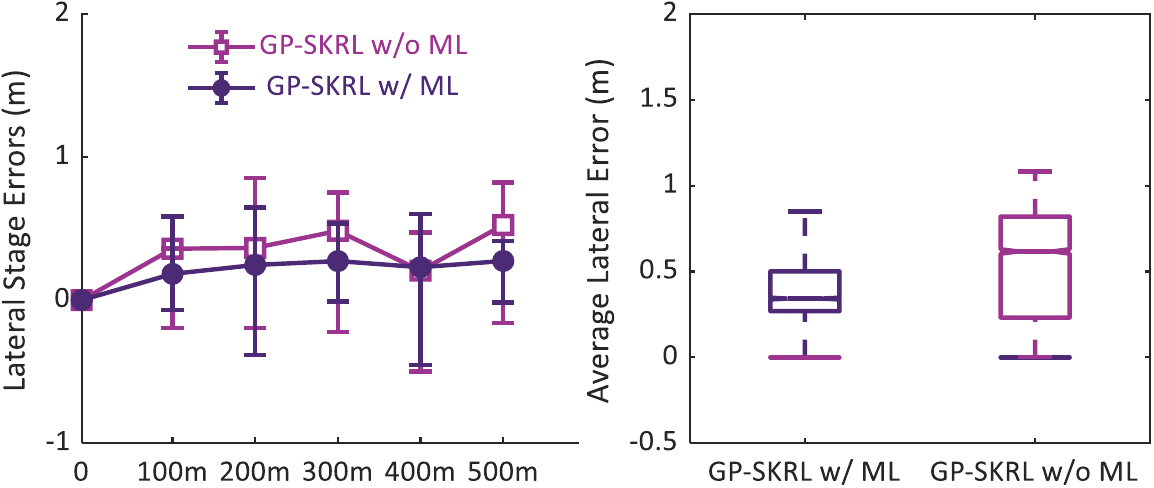}
\caption{Tracking error comparison in the racetrack road of Fig.~\ref{fig:CompModelLearning} between {GP-SKRL} (w/o ML) and {GP-SKRL} (w/ ML). }
\label{fig:Error_ComparisonTrackingWithAndWithoutModelLearning}
\end{figure}

\begin{table}[htbp]
\centering\setlength\tabcolsep{4.5pt}
\scriptsize
\caption{Comparison Between Two Sparse GP Methods.}
\label{tab:SpareseGP} 
\begin{threeparttable} 
	\begin{tabular}{ccccccc}
		\toprule
		\multicolumn{1}{c}{\multirow{2}{*}{{\textbf{{NDPs}}}}}	&\multicolumn{2}{c}{{\textbf{{training time \rm{(s)}}}}} & \multicolumn{2}{c}{{\textbf{APE in $v_y$} ($10^{-3}$ \rm{m/s})}} & \multicolumn{2}{c}{{\textbf{APE in $\omega$} ($10^{-3}$\rm{rad/s})}} \\ 
		\multicolumn{1}{c}{}                     &   \multicolumn{1}{c}{{ALD-GP}} & \multicolumn{1}{c}{{FITC-GP}}                          &             \multicolumn{1}{c}{{ALD-GP}} & \multicolumn{1}{c}{{FITC-GP}}             &       \multicolumn{1}{c}{{ALD-GP}} & \multicolumn{1}{c}{{FITC-GP}}                   \\ \midrule
		{$9000$}                                      & \multicolumn{1}{c}{{$0.10$}}  & {$43.14$}                       & \multicolumn{1}{c}{{$1.36$}}  & {$0.29$}                       & \multicolumn{1}{c}{{$1.05$}}  & {$0.30$}                       \\ 
		{$6000$}                                       & \multicolumn{1}{c}{{$0.05$}}  & {$14.65$}                       & \multicolumn{1}{c}{{$1.37$}}  & {$0.31$}                       & \multicolumn{1}{c}{{$1.05$}}  & {$0.31$}                       \\ 
		{$3000$}                                       & \multicolumn{1}{c}{{$0.04$}}  & {\;\;$2.71$}                       & \multicolumn{1}{c}{{$1.38$}}  & \;{$0.36$ }                      & \multicolumn{1}{c}{{$1.05$}}  & {$0.35$}                       \\ 
		{$1000$}                                       & \multicolumn{1}{c}{{$0.03$}}  & {\;\;$0.31$}                       & \multicolumn{1}{c}{{$2.86$}}  & {$1.76$}                       & \multicolumn{1}{c}{{$2.13$}}  & {$1.56$}                       \\ \bottomrule
	\end{tabular}
\end{threeparttable}
\end{table}

{To validate the effectiveness of ALD-GP (sparse GP with ALD), we conducted the following test. For a fixed number of data points (NDPs), training and testing were performed using ALD-GP and FITC-GP (sparse GP with FITC) methods separately. The resulting training time and the one-step average prediction error (APE) were recorded and are presented in TABLE IV. In summary, the ALD-GP method shows slightly higher average errors in $v_y$ and $\omega$, but the training time is significantly lower compared with FITC-GP when the value of NDPs is high. Therefore, ALD-GP significantly improves computational efficiency, sacrificing only little accuracy.}
\subsubsection{Validation of the Online Adaption Capability}\label{AdaptionCapability}
To highlight the online adaption capability of the algorithm, abbreviated as GP-SKRL w/ OA, we deployed the RL policy trained with the nominal model (abbreviated as GP-SKRL w/o OA) in a tracking control task. The exact vehicle parameters were set as different values in three stages: $0$-$170$ \rm{m} ($m=2257$ \rm{kg}, $I_z=3524.9$ \rm{kg}$\cdot$ \rm{m}), $170$-$340$ \rm{m} ($m=1957$ \rm{kg}, $I_z=3224.9$ \rm{kg}$\cdot$\rm{m}), and $340$-$508$ \rm{m} ($m=1657$ \rm{kg}, $I_z=2924.9$ \rm{kg}$\cdot$\rm{m}). The first half stages, i.e., $0$-$85$ m, $170$-$255$ m, and $340$-$425$ m, adopted the same policy as GP-SKRL w/o OA to collect training data. Online policy updates were performed at positions $85$ m, $255$ m, and $425$ m, with an average time of $0.99$ seconds. From the simulation results in Fig.~\ref{fig:adaption_com_nominal_results}, the tracking control error is reduced in the second half stages, $85$-$170$ m, $255$-$340$ m, and $425$-$508$ m. The simulation results validate the online adaption capability of our algorithm.
\begin{figure}[!ht]
\centering\includegraphics[width=3.0in]{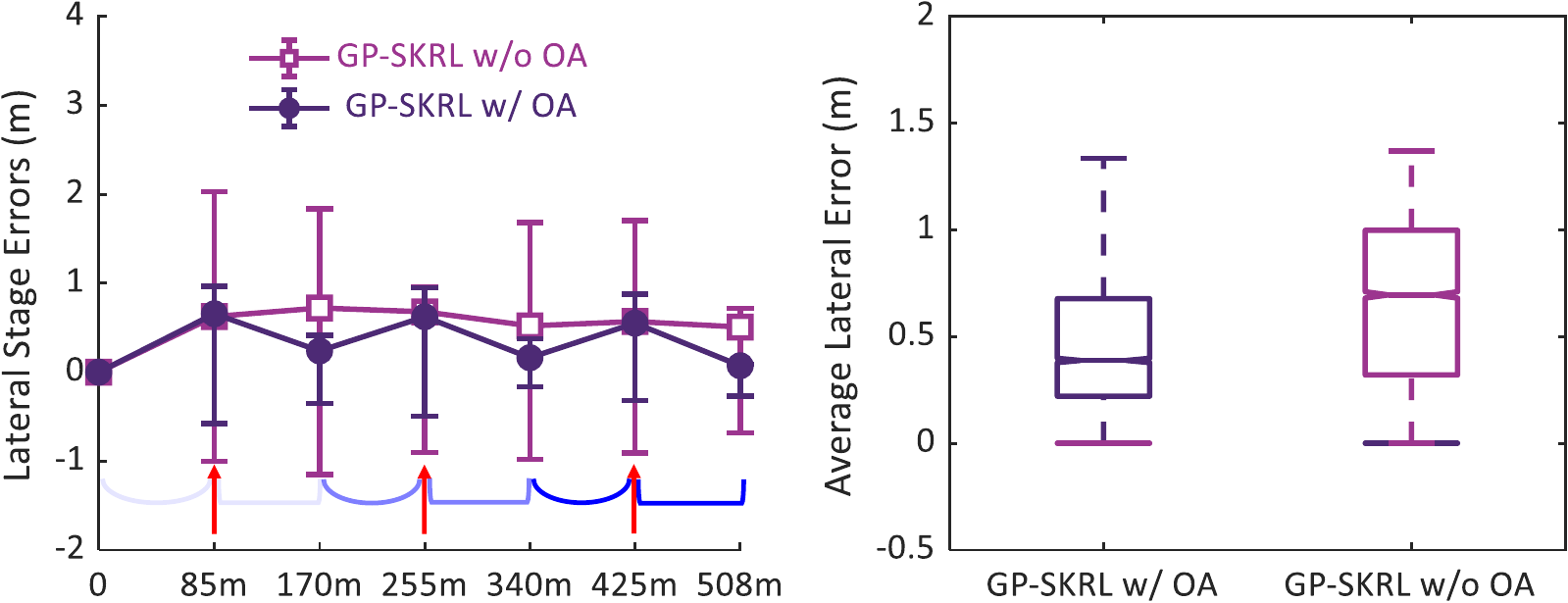}
\caption{Tracking error comparison in the racetrack road of Fig.~\ref{fig:CompModelLearning} between {GP-SKRL} (w/o OA) and {GP-SKRL} (w/ OA). The red arrows indicate the positions where RL policies were updated. }
\label{fig:adaption_com_nominal_results}
\end{figure}

\subsection{Real-World Experiments}
The experimental platform, shown in the central image of Fig.~\ref{fig_2},  is a real Hongqi E-HS3 electric car equipped with a positioning system and perception sensors, including millimeter-wave radar, LiDAR, and cameras.
The planning program was run on a laptop in the Windows operating system with an Intel i7-11800H CPU @2.30GHz, and the computed control sequence was transferred to the low-level control module using the Robot Operating System (ROS)\cite{quigley2009ros}.

\subsubsection{Case 1: Experimental comparisons with optimization-based approaches}
We further compared {GP-SKRL} with several optimization-based approaches in an unstructured environment. With an initial velocity of $0$ \rm{m/s} and a desired velocity of $2$ \rm{m/s}, the ego vehicle has to avoid the moving obstacle to remain safe. For comparison, the prediction horizons of LMPCC and MPC-CBF were both set to $20$. To improve the computational efficiency, a successive linearization strategy was adopted in our implementation, and the OSQP\cite{stellato2020osqp} solver was used. Consistent with the error definitions of the lateral and longitudinal directions in \cite{brito2019model}, the cost function was set to $\mathcal{Q}_1e_x^2+\mathcal{Q}_2e_y^2+\mathcal{Q}_3e_{\varphi}^2+\mathcal{R}_1^2\delta_{f}+\mathcal{R}_2^2a_x+P_{\text{colli}}$, where the state penalty matrix $[\mathcal{Q}_1,\mathcal{Q}_2,\mathcal{Q}_3]$ was set to $[10,5,500]$, the control matrix $\mathcal{R}$ was set to $\text{diag}\{50,360\}$. If a collision with obstacles occurs, then $P_{\text{colli}}=100$; otherwise, $P_{\text{colli}}=0$. 

Two static obstacles are placed on the reference trajectory while a moving obstacle is crossing the road. 
\begin{figure}[!ht]
\centering
\centering\includegraphics[width=3.5in]{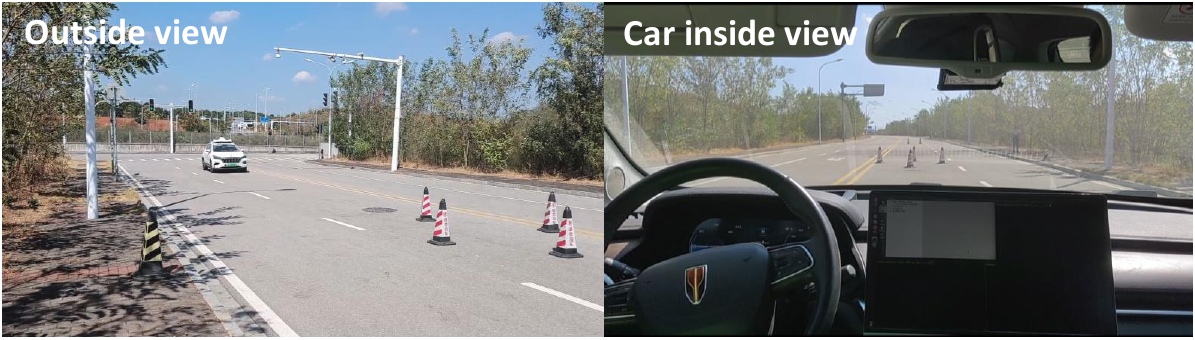}
\caption{Views from inside and outside the car during the experiment. For simplicity, we ignore the presence of lanes and it is assumed safe for the vehicle to drive into the opposite lane in this experiment.}
\label{fig_ExperimentPicture}
\end{figure}

The vehicle travels from the starting point $(0,0)$ to the end point $(150,0)$. It must avoid two static obstacles and a moving pedestrian to maintain safe driving. As seen in Fig.~\ref{ExpComparisionsResults}, the final driving trajectories using the LMPCC, {GP-SKRL} and MPC-CBF planners are depicted with the green, blue, and red gradient-colour lines, respectively. LMPCC can successfully avoid the first two static obstacles, but it does not take any actions due to the failure to find a feasible solution. MPC-CBF successfully avoids the two grey static obstacles. As the pedestrian crosses the road, the ego vehicle must engage in emergency obstacle-avoidance behaviour, but the vehicle fails to avoid collision in time, causing a solution failure of the optimization problem, which is mainly attributed to the heavy computational load. Notably, {GP-SKRL} successfully completes the tasks in real-time. 
\begin{figure}[!htb]
\centering
\centering\includegraphics[width=3.5in]{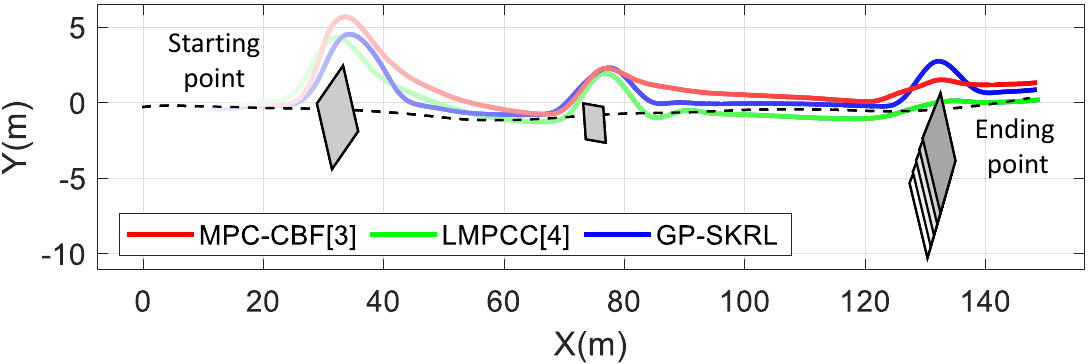}
\caption{Experimental results of three approaches for obstacle avoidance tasks. The polygons with black contours are obstacles blocking the road. The moving pedestrian is represented by
	multiple grey polygons. All approaches are required to have a desired speed of $2$ \rm{m/s}.}
\label{ExpComparisionsResults}
\end{figure}
\begin{table}[!htb]\centering
\setlength\tabcolsep{12pt}
\scriptsize
\caption{Performance Evaluations of Fig.~\ref{ExpComparisionsResults}}\label{table:ExpQuantativeResults}
\begin{threeparttable}
	\begin{tabular}{ccccccccc}
		\toprule
		{Quantitative Metrics}&{{GP-SKRL}}&{MPC-CBF$^*$}&{LMPCC$^*$}\\ \midrule
		$J$&$\mathbf{62.18}$&$185.77$&$161.58$\\
		{Aver. S.T.}&$\mathbf{0.005}$ \rm{ms}&$90$ \rm{ms}&$100$ \rm{ms}\\ \bottomrule
	\end{tabular}
	\begin{tablenotes}
		\footnotesize
		\item[*] MPC-CBF and LMPCC with the kinodynamic constraint fail to solve the planning problem in time. 
	\end{tablenotes}
\end{threeparttable}
\end{table}
\subsubsection{Case 2: Moving obstacle avoidance}
In this case, a moving polygonal obstacle moves along the reference trajectory. With an initial velocity of $0$ \rm{m/s} and a desired velocity of $5$ \rm{m/s}, the ego vehicle has to avoid the moving obstacle to remain safe. As seen in Fig.~\ref{Exp_ExperimentalResultsPlot1}, the ego vehicle (blue boxes) starts from the coordinate origin and tracks the reference trajectory. Before crashing into a blocking moving obstacle (grey polygons), the ego vehicle proactively avoids it with a lateral displacement and finally tracks the reference trajectory.

\begin{figure}[!ht]
\centering
\centering\includegraphics[width=3.5in]{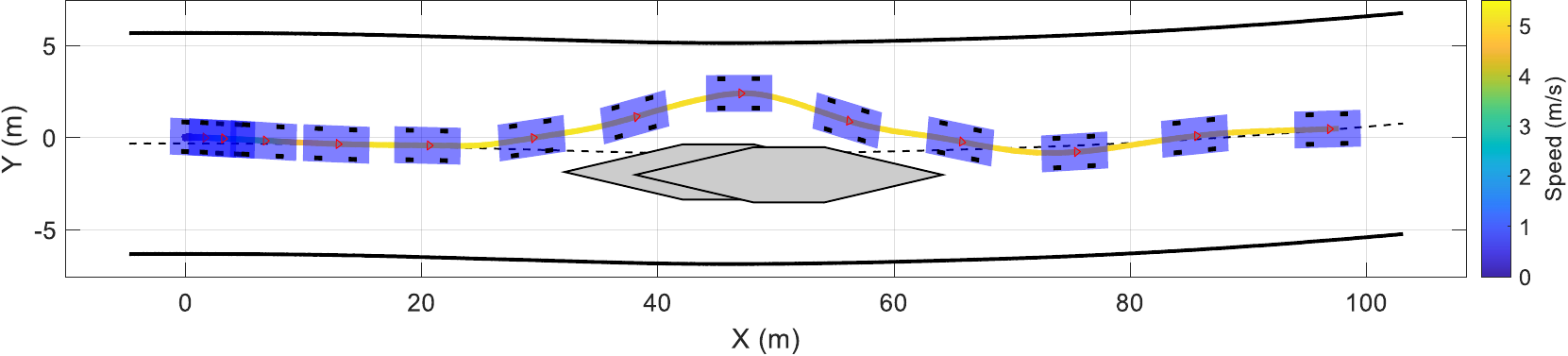}
\caption{Moving obstacle avoidance where the black dashed line is the reference trajectory, and the thick black line is the road boundary.}
\label{Exp_ExperimentalResultsPlot1}
\end{figure}
\subsection{Analysis and Discussion}
{As mentioned in Sec.~\ref{SKRL}, the computational complexity of GP-SKRL is $\mathcal{O}(M{n_{\mathcal{K}}^2})$. The complexity increases exponentially as the values of $M$ and $n_\mathcal{K}$ increase. For methods without using the quantized sparse technique, the complexity is $\mathcal{O}(M^3)$, which is much higher. In our implementation, GP-SKRL converges after about $60$ training iterations with 30000 samples.  And each iteration only takes around $0.015$\rm{s}.}

From the comparison results, several factors contribute to the superiority of {GP-SKRL}: 1) The proposed sparse GP can generate representative samples using a quantized sparse technique, thus outperforming the FITC method in terms of computational efficiency; 2) We propose an efficient sparse kernel-based RL method for learning optimal motion planning policies, enabling online adaption with near-optimal performance; 3) By incorporating a safety-aware module, GP-SKRL is able to assess the feasibility of generated local trajectories, and therefore, it enables a vehicle to return to the reference trajectory as soon as possible. GP-SKRL ultimately achieves performance surpassing that of several advanced methods.

\section{Conclusion}\label{Conclusion}

This paper proposed a learning-based near-optimal motion planning approach with online adaption capability for intelligent vehicles with uncertain dynamics. The sparse kernel-based technique is manifested not only in the basis functions of RL but also in GPs. Therefore, the RL policies can be efficiently learned to guide the vehicle to maintain a safe distance from obstacles. We demonstrated the effectiveness of GP-SKRL in completing motion planning tasks via simulation and real-world experimental results. In particular, we demonstrated that the performance in a racing scenario can be significantly improved by using the sparse GP technique.
\appendix
\subsection{Vehicle Dynamics}\label{nominal_dynamics}
The vehicle dynamics are described using a ``bicycle'' model in \cite{rajamani2011vehicle}, and the nominal dynamics are expressed as
\begin{footnotesize}
	\begin{equation*}\label{vehicle_dynamics}
		\left[ \begin{array}{l}
			\dot{v_x}\\
			\dot{v_y}\\
			\dot{\varphi}\\
			\dot{\omega}\\
			\dot{X}\\
			\dot{Y}\\
		\end{array} \right] =\underbrace{\left[ \begin{array}{c}
				v_y\omega +a_x\\
				2C_{af}( \frac{\delta_f}{m} -\frac{v_y+l_f\omega}{mv_x}) + 2C_{ar}\frac{l_r\omega -v_y}{mv_x}  -v_x\omega\\
				\omega\\
				\frac{2}{I_z}\left[ l_fC_{af}( \delta _f-\frac{v_y+l_f\omega}{v_x}) -l_rC_{ar}\frac{l_r\omega -v_y}{v_x} \right]\\
				v_x\cos \varphi -v_y\sin \varphi\\
				v_x\sin \varphi +v_y\cos \varphi\\
			\end{array} \right]}_{f_{\text{nom}}^0(x,u)}, 
	\end{equation*}
\end{footnotesize}where $x=(v_x, v_y, \varphi, \omega, X, Y)\in \mathbb{R}^6$ is the state vector, $v_x, v_y$ denote the longitudinal and lateral velocities, respectively, $\varphi$ is the yaw angle, $\omega$ denotes the yaw rate, $X, Y$ are the global horizontal and vertical coordinates of the vehicle, respectively, $l_f, l_r$ are the distances from the center of gravity (CoG) to the front and rear wheels, respectively, $C_{af}, C_{ar}$ represent the cornering stiffnesses of the front and rear wheels, respectively, and $I_z$ denotes the yaw moment of inertia. In this model, acceleration $a_x$ and steering angle $\delta_{f}$ are two variables of the control vector $u$, i.e., $u=[a_x, \delta_{f}]^{\top}$. The reference state is denoted by ${x_r}=(v_x^r, v_y^r, \varphi_r, \omega_r, X_r, Y_r)$. Fig.~\ref{fig_1} displays the relationship between the driving vehicle and the projection point $x_r=(v_x^r, v_y^r, \varphi_r, \omega_r, X_r, Y_r)$. To construct an error model for facilitating subsequent algorithm design, we subtract the current state from the desired state, i.e, $\mathbf{x}=x-x_r$ and the system control from desired control, i.e., $\mathbf{u} = u-u_r = [a_x,\delta_{f}]^{\top}$, where $u_r = [0, 0]$.
\begin{figure}[htb]
	\centering\includegraphics[width=3.0in]{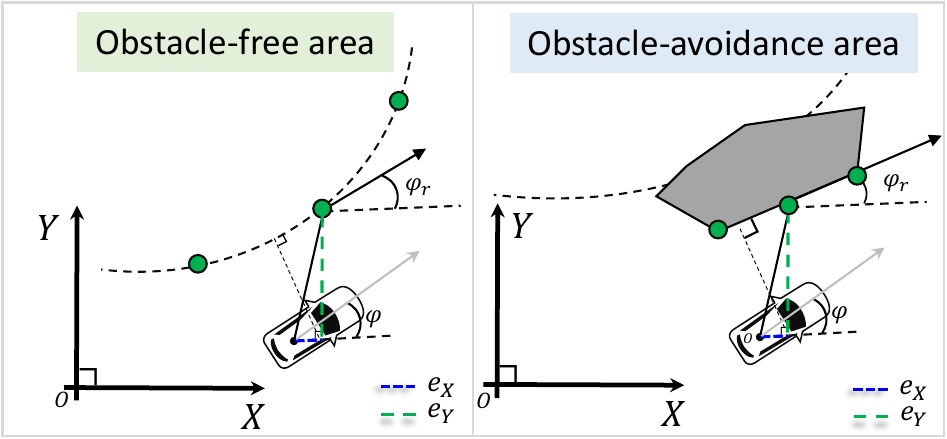}
	\caption{The relationship between the vehicle and the reference points $x_r$. The green circles are the reference points, while the distance errors in the $X$ and $Y$ directions are blue and green dashed lines, respectively.	}
	\label{fig_1}
\end{figure}
\subsection{Parameter Settings of the Training Process}
The state boundaries in sampling and training are set as follows: $e_{v_x}\in[-6,6]$ \rm{m/s}, $e_{v_y}\in[-6,6]$ \rm{m/s},  $e_{\varphi}\in[-\frac{\pi}{3},\frac{\pi}{3}]$ \rm{rad},  $e_{\omega}\in[-6,6]$ \rm{rad/s}, $e_X \in[-3,3]$ \rm{m}, $e_Y\in[-3,3]$ \rm{m}.
The control boundaries are set as follows: $a_x\in \left[ -1,1\right]$ \rm{m/s}$^2$ and $\delta _f\in \left[ -{\pi}/{6} , {\pi}/{6}\right]$ \rm{rad}.
The real vehicle parameters for the simulation are set as follows: $M=2257$ \rm{kg}, $l_f=1.33$ \rm{m}, $l_r=1.81$ \rm{m}, $C_{af}=60790$ \rm{N/rad}, $C_{ar}=50400$ \rm{N/rad}, $g=9.8$ \rm{m/$\rm{s}^2$}, $I_z=3524.9$ \rm{kg$\cdot$m}. Gaussian kernel function is adopted, i.e., $k(s_i,s_j)=\text{exp}(-\lVert s_i-s_j \rVert^2/\tau^2)$. The kernel width $\tau$ used to construct the dictionary is $0.9$, and a vector $\tau=[0.7,0.8,0.9,1]$ is adopted to construct the multikernel feature. The penalty coefficient $\mu$ is set to $6$ for an obstacle-avoidance policy and $\mu=0$ for a control policy. 
\subsection{Simulation Settings of Comparison Approaches}\label{Appendix_Settings}
Comparison approaches, including MPC-CBF, LMPCC, and CFS, use the IPOPT\cite{wachter2006implementation} solver. They all consider the nonlinear vehicle dynamics~\eqref{vehicle_dynamics} to compare the performance. Kd-RRT$^\star$ approach solves the planning problem through iterative optimization. For CFS, we used the obstacle-avoidance idea in \cite{liu2018convex} to establish the obstacle constraints and implemented a kinodynamic motion planning approach; see Sec. IV-E.  For MPC-CBF, the extended class $\mathcal{K}_{\infty}$ function $\gamma$ is set to $0.4$ in the first scenario and $0.6$ in the second one. For LMPCC, the potential function is included in the cost, i.e., 
\begin{equation*}
	\begin{array}{c}
		J_{\text {repulsive }}=Q_{R} \sum_{k=1}^{N_\text{obs}}\left(\frac{1}{\left(\Delta x_{k}\right)^{2}+\left(\Delta y_{k}\right)^{2}+\epsilon}\right),
	\end{array}
\end{equation*}
where the coefficient $Q_R$ is set to $5000$ in the first scenario and $30000$ in the second scenario, $N_\text{obs}$ is the number of  moving obstacles, $\Delta x_k,\Delta y_k$ represents the distance from the vehicle to the center of the $k$-th moving obstacle, and the parameter $\epsilon$ is set to $0.001$. In Scenario II, the first elliptical obstacle moves if $X>30$, while the second obstacle moves if $X>80$. The simulation was performed on a laptop running MATLAB 2020a with an i7-11800H CPU @2.30GHz without parallel acceleration.
\subsection{Proof of Theorem~\ref{thm1}}\label{AppendixProof}
	(Proof of Theorem~\ref{thm1}): According to Eq.~\eqref{target_lambda}, the costate variable at the $i$-th iteration can be computed by
	\vspace{1mm}
	\begin{equation*}
		\begin{array}{c}
			\lambda_k^{[i]}= 2Q\mathbf{x}_k+\gamma A_{d,k}^{\top}\frac{\partial V_i\left({\mathbf{x}}_{k+1} \right)}{\partial {\mathbf{x}}_{k+1}}+\mu\frac{\partial{\mathcal{B}({x}_k)}}{\partial{\mathbf{x}_k}},
		\end{array}
	\end{equation*}
	where $\mu\frac{\partial{\mathcal{B}({x}_k)}}{\partial{\mathbf{x}_k}}$ is bounded.
	As $i\rightarrow\infty$, the costate variable $\lambda_{\infty}$ can be written as
	\begin{equation*}\label{Lambda_infty}
		\begin{array}{c}
			\lambda_{\infty}(\mathbf{x}_k)= 2Q\mathbf{x}_k+\gamma A_{d,k}^{\top}\frac{\partial V_{\infty}\left({\mathbf{x}}_{k+1} \right)}{\partial {\mathbf{x}}_{k+1}}+\mu\frac{\partial{\mathcal{B}({x}_k)}}{\partial{\mathbf{x}_k}}.
		\end{array}
	\end{equation*}
	
	{Using Lemma in \cite{al2008discrete}, as $i\rightarrow\infty$, $V_{\infty}(\mathbf{x}_k)\rightarrow V^*(\mathbf{x}_k)$. The costate $\lambda_{\infty}(\mathbf{x}_k)$ can be rewritten as
	\begin{equation*}
		\begin{aligned}
			\lambda_{\infty}(\mathbf{x}_k)&= \begin{array}{c}2Q\mathbf{x}_k+\gamma A_{d,k}^{\top}\frac{\partial V^*\left( {\mathbf{x}}_{k+1} \right)}{\partial {\mathbf{x}}_{k+1}}+\mu\frac{\partial{\mathcal{B}({x}_k)}}{\partial{\mathbf{x}_k}}\end{array}\\
			&=\begin{array}{c}2Q\mathbf{x}_k+\gamma A^{\top}_{d,k}\lambda^*({\mathbf{x}}_{k+1})+\mu\frac{\partial{\mathcal{B}({x}_k)}}{\partial{\mathbf{x}_k}}
			\end{array}=\lambda^{*}(\mathbf{x}_k).
		\end{aligned}
	\end{equation*}}
	
	{That is, $\lambda_k^{[i]}$ will converge to $\lambda^{*}_k$, as $i\rightarrow\infty$.
	After we substitute the optimal costate $\lambda^*$ into~\eqref{optimal_action}, it follows that $\mathbf{u}^{[i]}\rightarrow \mathbf{u}^*$ as $i\rightarrow\infty$. As the number of iterations goes to infinity, $\hat{\mathbf{\Lambda}}^{[i]}\rightarrow\mathbf{\Lambda}^*$ and $\hat{\mathbf{U}}^{[i]}\rightarrow \mathbf{U}^*$. $\hfill\blacksquare$}
	\bibliographystyle{IEEEtran}
\bibliography{IEEEabrv.bib}
\end{document}